\documentclass{article}

\PassOptionsToPackage{numbers, compress}{natbib}
\usepackage[preprint]{neurips_2025}
\usepackage{amsmath, amssymb, amsthm}

\usepackage{fancyhdr}

\usepackage[utf8]{inputenc} 
\usepackage[T1]{fontenc}    
\usepackage{hyperref}       
\usepackage{url}            
\usepackage{booktabs}       
\usepackage{amsfonts}       
\usepackage{nicefrac}       
\usepackage{microtype}      
\usepackage{xcolor}         
\usepackage{graphicx} 
\usepackage{subcaption}
\usepackage{wrapfig}
\usepackage{hyperref}
\usepackage{algorithm}
\usepackage{algpseudocode}

\usepackage{amsmath, amssymb, amsthm}
\usepackage{amsfonts} 

\newtheorem{theorem}{Theorem}[section]
\newtheorem{definition}{Definition}[section]
\newtheorem{proposition}{Proposition}[section]

\theoremstyle{remark}

\newcommand{\Istep}{I_s}
\newcommand{\Itotal}{I_{\text{total}}}
\newcommand{\Cstep}{C_s}
\newcommand{\Ceff}{C_{\text{effective}}}
\newcommand{\Thetagoal}{\Theta_{\text{goal}}}
\newcommand{\argmax}{\operatornamewithlimits{argmax}}
\newcommand{\Ceffective}{C_{\mathrm{effective}}}

\title{The Agent Capability Problem: Predicting Solvability Through Information-Theoretic Bounds}

\author{%
  Shahar Lutati \\
  Blavatnik School of Computer Science \\
  Tel Aviv University \\
  \texttt{shahar761@gmail.com} \\
}
\begin{document}

\maketitle

\maketitle

\begin{abstract}
When should an autonomous agent commit resources to a task? We introduce the Agent Capability Problem (ACP), a framework for predicting whether an agent can solve a problem under resource constraints. Rather than relying on empirical heuristics, ACP frames problem-solving as information acquisition: an agent requires $\Itotal$ bits to identify a solution and gains $\Istep$ bits per action at cost $\Cstep$, yielding an effective cost $\Ceff = (\Itotal/\Istep), \Cstep$ that predicts resource requirements before search. We prove that $\Ceff$ lower-bounds expected cost and provide tight probabilistic upper bounds. Experimental validation shows that ACP predictions closely track actual agent performance, consistently bounding search effort while improving efficiency over greedy and random strategies. The framework generalizes across LLM-based and agentic workflows, linking principles from active learning, Bayesian optimization, and reinforcement learning through a unified information-theoretic lens.
\end{abstract}

\section{Introduction}

Resource allocation remains a stubborn challenge for autonomous agents. An agent facing a new problem must decide whether attempting a solution is worthwhile, yet most frameworks lack principled methods for making this judgment before substantial resources are spent.

We propose the Agent Capability Problem (ACP): given a problem's structure and an agent's capabilities, can we predict whether a solution exists within budget? This question matters because failed attempts waste resources that could be allocated elsewhere, while conservative estimates leave valuable problems unsolved.

Our approach treats problem-solving as information search. A solution exists somewhere in hypothesis space $\Theta$, and each action reduces uncertainty about its location. This perspective suggests a natural cost model: the information needed ($\Itotal$) divided by the information gained per action ($\Istep$), multiplied by cost per action ($\Cstep$). The resulting quantity, $\Ceff$, serves two purposes. First, it provides an upfront solvability estimate, if $\Ceff$ exceeds budget $B$, the agent should reconsider. Second, it guides action selection during search by prioritizing high information-to-cost ratios.

The ACP formalizes intuitions that appear across machine learning. Active learning selects maximally informative labels \cite{mackay1992, houlsby2011}. Bayesian optimization queries points that reveal the most about an optimum's location \cite{hennig2012, hernandezlobato2014}. Curious reinforcement learning agents seek surprising observations \cite{schmidhuber2010}. Each instantiates the same principle: maximize information gain relative to cost. The ACP makes this principle explicit and shows how to operationalize it for solvability assessment.

Section 2 formalizes the ACP and defines its core quantities. Section 3 proves that $\Ceff$ lower-bounds expected cost and derives additive and probabilistic upper bounds. Section 4 reviews how mutual information appears in related research areas. Section 5 describes practical approximation using Gaussian processes and validates predictions empirically. Section 6 extends the framework to approximation algorithms for NP-hard problems.

\section{Formal Framework}

Let $\Theta$ denote the space of candidate solutions. A subset $\Thetagoal \subseteq \Theta$ contains acceptable solutions. An agent observes the world by taking actions $a \in \mathcal{A}$, each producing an outcome $y \in \mathcal{Y}$ at cost $\Cstep(a)$.

\begin{definition}[Solvability Criterion]
An agent with budget $B$ can solve a problem if
\begin{equation}
B \ge \Ceff,
\label{eq:acp_core}
\end{equation}
where $\Ceff$ is the expected cost under an optimal information-gathering policy.
\end{definition}

This criterion depends on three quantities. The first measures initial uncertainty.

\begin{definition}[Total Information Requirement]
Let $\mathbf{1}(\theta \in \Thetagoal)$ indicate whether candidate $\theta$ satisfies the goal. The total information needed is
\begin{equation}
\Itotal = H(\mathbf{1}(\theta \in \Thetagoal)),
\end{equation}
the entropy of the goal indicator under the prior.
\end{definition}

When $\Thetagoal$ contains a fraction $p$ of candidates, $\Itotal = -p\log p - (1-p)\log(1-p)$, maximized when $p=1/2$. Problems with rare solutions require more information than those with abundant solutions.

\begin{definition}[Information Gain per Action]
An action $a$ producing outcome $y$ provides information
\begin{equation}
\Istep(a) = I(\mathbf{1}(\theta \in \Thetagoal); y | a),
\end{equation}
the mutual information between the outcome and goal indicator.
\end{definition}

This quantity captures how much an observation narrows the search. Uninformative actions yield $\Istep \approx 0$ while decisive tests achieve $\Istep \approx \Itotal$.

\begin{definition}[Optimal Action Selection]
At each step, select
\begin{equation}
a^* = \argmax_{a \in \mathcal{A}} \left( \frac{\mathbb{E}[\Istep(a)]}{\Cstep(a)} \right).
\end{equation}
\end{definition}

This policy maximizes information gained per resource spent, analogous to economic efficiency. The effective cost is the expected total expenditure under this policy.

\begin{definition}[Effective Cost]
\begin{equation}
\Ceff = \frac{\Itotal}{\bar{\Istep}} \times \bar{\Cstep},
\label{eq:ceff}
\end{equation}
where $\bar{\Istep}$ and $\bar{\Cstep}$ are average values under the optimal policy.
\end{definition}

When actions have uniform cost, $\Ceff$ simplifies to the number of steps times cost per step. The subsequent section proves this estimate is achievable.

\section{Theoretical Bounds}

We model information accumulation as a sequential stopping problem and derive bounds on total cost. These bounds reveal when $\Ceff$ accurately predicts resource requirements and when uncertainty inflates actual costs.

\subsection{Stochastic Model}

Let $\{X_i\}_{i \ge 1}$ represent information gains from successive actions. We assume:
\begin{enumerate}
\item Independence: outcomes are conditionally independent given history
\item Bounded moments: $\sup_i \mathbb{E}[X_i^2] < \infty$
\item Diminishing returns: expected gains satisfy $\mu_1 \ge \mu_2 \ge \dots \ge \mu_{\inf} > 0$
\end{enumerate}

The third assumption reflects that as search proceeds, remaining uncertainty becomes harder to resolve. Early actions eliminate large swaths of hypothesis space; later actions refine estimates. We set $\Istep := \mu_1$, the initial expected gain.

Define the stopping time
\begin{equation}
N := \inf\left\{n \ge 1 : \sum_{i=1}^n X_i \ge \Itotal\right\}
\end{equation}
and total cost $C := \Cstep N$. Our goal is bounding $\mathbb{E}[C]$.

\subsection{Main Result}

\begin{theorem}[Two-Sided Cost Bound]
\label{thm:twosided_adaptive}
Under the assumptions above,
\begin{equation}
\Ceffective \;\le\; \mathbb{E}[C] \;\le\; \Cstep \left( \frac{\Itotal}{\mu_{\inf}} + \frac{M_2}{\mu_{\inf}^2} \right),
\end{equation}
where $M_2 = \sup_i \mathbb{E}[X_i^2]$ and $\mu_{\inf} = \inf_i \mu_i$.
\end{theorem}

\begin{proof}
The partial sums minus their expectations, $M_n = S_n - \sum_{i=1}^n \mu_i$, form a martingale. Assuming finite expected stopping time, the optional stopping theorem gives $\mathbb{E}[M_N] = 0$, hence
\begin{equation}
\mathbb{E}[S_N] = \mathbb{E}\left[\sum_{i=1}^N \mu_i\right].
\end{equation}

For the lower bound, $S_N \ge \Itotal$ by definition, so $\mathbb{E}[S_N] \ge \Itotal$. Since $\mu_i \le \Istep$ for all $i$,
\begin{equation}
\mathbb{E}\left[\sum_{i=1}^N \mu_i\right] \le \Istep \mathbb{E}[N].
\end{equation}
Combining these: $\Itotal \le \Istep \mathbb{E}[N]$, giving $\mathbb{E}[N] \ge \Itotal/\Istep$ and thus $\mathbb{E}[C] \ge \Ceffective$.

For the upper bound, let $R := S_N - \Itotal$ denote overshoot. Then $\mathbb{E}[S_N] = \Itotal + \mathbb{E}[R]$. Since $\mu_i \ge \mu_{\inf}$,
\begin{equation}
\mathbb{E}\left[\sum_{i=1}^N \mu_i\right] \ge \mu_{\inf} \mathbb{E}[N].
\end{equation}
Lorden's inequality for independent non-identically distributed variables bounds expected overshoot by $\mathbb{E}[R] \le M_2/\mu_{\inf}$ \cite{lorden1970}. Therefore
\begin{equation}
\mu_{\inf}\mathbb{E}[N] \le \Itotal + \frac{M_2}{\mu_{\inf}},
\end{equation}
which rearranges to the stated upper bound.
\end{proof}

This result has practical implications. The lower bound confirms that $\Ceffective$, computed using the most optimistic gain $\Istep$, is realistic, no strategy can do better on average. The upper bound warns that if information gains deteriorate severely ($\mu_{\inf} \ll \Istep$), costs can exceed initial estimates substantially. The gap between bounds depends on $M_2/\mu_{\inf}^2$, a signal-to-noise ratio.

\subsection{High-Probability Bound}

When information increments are bounded, $0 \le X_i \le M$ almost surely, we can derive probabilistic guarantees. Hoeffding's inequality for independent bounded variables gives
\begin{equation}
\Pr(S_n < \Itotal) \le \exp\left( -\frac{2(n\mu_{\inf} - \Itotal)^2}{n M^2} \right)
\end{equation}
for $n\mu_{\inf} > \Itotal$. Setting the right side to $\delta$ and solving for $n$ yields a step count sufficient to finish with probability at least $1-\delta$:
\begin{equation}
n_\delta \ge \frac{\Itotal}{\mu_{\inf}} + \frac{M^2}{2\mu_{\inf}^2}\log\frac{1}{\delta} + \sqrt{\frac{\Itotal M^2}{2\mu_{\inf}^3}\log\frac{1}{\delta}}.
\end{equation}

This gives a high-confidence budget estimate:
\begin{equation}
\label{eq:highprob}
C \le \Cstep\left(\frac{\Itotal}{\mu_{\inf}} + O\left(\frac{M^2}{\mu_{\inf}^2}\log\frac{1}{\delta}\right)\right) \quad\text{with probability } \ge 1-\delta.
\end{equation}

For mission-critical applications where exceeding budget has severe consequences, this bound allows setting reserves proportional to $\log(1/\delta)$.

\section{Connections to Information-Theoretic Search}

The ACP's use of mutual information is not novel. Multiple research communities have independently converged on this quantity as the right measure for guiding search. We briefly review these connections.

\subsection{Active Learning}

Active learning addresses label acquisition when labeling is expensive. MacKay \cite{mackay1992} proposed selecting queries that maximize expected information gain about model parameters. Houlsby et al. \cite{houlsby2011} developed Bayesian Active Learning by Disagreement (BALD), whose acquisition function is exactly the mutual information between predictions and parameters. Classical experimental design criteria like D-optimality \cite{chaloner1995}, which maximize the determinant of the Fisher information matrix, similarly aim to maximize parameter information.

\subsection{Bayesian Optimization}

Bayesian optimization seeks global optima of expensive black-box functions. While early acquisition functions like Upper Confidence Bound \cite{srinivas2010} balance exploration and exploitation heuristically, Entropy Search (ES) \cite{hennig2012} and Predictive Entropy Search (PES) \cite{hernandezlobato2014} directly maximize information about the optimum's location. The ES acquisition function
\begin{equation}
\alpha_{\text{ES}}(x) = H[p(x^*|\mathcal{D})] - \mathbb{E}_{y}[H[p(x^*|\mathcal{D} \cup \{(x,y)\})]]
\end{equation}
is mutual information between observation $y$ at $x$ and optimum $x^*$. Regret bounds for Gaussian process optimization \cite{srinivas2010} explicitly depend on maximum information gain, proving that efficient optimization requires efficient information acquisition.

\subsection{Intrinsic Motivation in Reinforcement Learning}

Sparse reward environments challenge reinforcement learning agents. Intrinsic motivation provides auxiliary signals. Schmidhuber's curiosity framework \cite{schmidhuber2010} rewards prediction errors, driving agents toward surprising states that improve world models. Empowerment \cite{klyubin2005} defines intrinsic value as the channel capacity between actions and future states, an agent seeks configurations where it can exert maximal influence, quantified information-theoretically.





\subsection{Estimation Procedure}
The estimation algorithm is given as pseudo-code to demonstrate the fundemnetal usage of ACP framework.
\begin{algorithm}
\caption{A Priori Cost Estimation}
\label{alg:acp_apriori}
\begin{algorithmic}[1]
\State \textbf{Input:} Hypothesis space $\Theta$, actions $\mathcal{A}$, costs $\{\Cstep(a)\}$, budget $B$
\State \textbf{Output:} Estimated cost $\widehat{\Ceff}$

\State \textbf{1. Model uncertainty:} Represent the unknown function $f:\Theta \to \mathbb{R}$ with Gaussian process prior $f \sim \mathcal{GP}(0, k(\cdot,\cdot))$

\State \textbf{2. Compute total information:} Approximate $\Itotal$ as entropy of the GP-induced distribution over $\Thetagoal$:
\begin{equation}
\Itotal = H[p(\theta \in \Thetagoal)] = -\int_{\Theta} p(\theta)\log p(\theta)\,d\theta
\end{equation}

\State \textbf{3. Estimate per-step gain:} For each action $a$, simulate outcomes $y \sim p(y|a)$, compute posterior entropies, and calculate
\begin{equation}
\text{IG}(a) = H[p(\theta)] - \mathbb{E}_{y|a}[H[p(\theta|a,y)]]
\end{equation}
Average over near-optimal actions to get $\widehat{\Istep}$

\State \textbf{4. Predict cost:}
\begin{equation}
\widehat{\Ceff} = \frac{\Itotal}{\widehat{\Istep}} \times \bar{\Cstep}
\end{equation}
If $\widehat{\Ceff} \le B$, predict solvability; otherwise reconsider
\end{algorithmic}
\end{algorithm}

Estimation errors arise from two sources: finite sampling when averaging over candidate actions, and surrogate model misspecification.

\begin{proposition}[Monte Carlo Error]
\label{prop:mc}
If $0 \le I_{\text{GP}}(\theta) \le L$ for all candidates and $\widehat{\Istep}$ averages over $S$ samples, then with probability at least $1-\delta$,
\begin{equation}
|\widehat{\Istep} - \mathbb{E}_{\theta}[I_{\text{GP}}(\theta)]| \le L\sqrt{\frac{\log(2/\delta)}{2S}}.
\end{equation}
\end{proposition}

\begin{proof}
Direct application of Hoeffding's inequality to bounded variables \cite{hoeffding1963}.
\end{proof}

For Gaussian predictive models, $L = \frac{1}{2}\log(1 + \sigma_{\max}^2/\sigma_n^2)$ where $\sigma_{\max}^2$ is maximum predictive variance.

Model error depends on whether the true function lies in the reproducing kernel Hilbert space (RKHS) of the GP kernel. Under standard regularity conditions \cite{srinivas2010, rasmussen2006}, GP posteriors concentrate around the truth at rate controlled by $\beta_t = O(\log t)$ and maximum information gain $\gamma_t$.

\begin{theorem}[Surrogate Error Under RKHS]
\label{thm:rkhs_error}
If $f \in \mathcal{H}_k$ with $\|f\|_{\mathcal{H}_k} \le B$, then with probability at least $1-\delta$,
\begin{equation}
|I_{\text{true}}(\theta) - I_{\text{GP}}(\theta)| \le \frac{1}{2(\sigma_n^2 + \underline{\sigma}^2)} \cdot \Delta_\sigma(\theta),
\end{equation}
where $\underline{\sigma}^2 = \inf_{\theta}\sigma_t^2(\theta)$ and $\Delta_\sigma(\theta) = O(\sqrt{\beta_t}\sigma_t(\theta))$.
\end{theorem}

\begin{proof}[Sketch]
Mutual information $g(v) = \frac{1}{2}\log(1+v/\sigma_n^2)$ is Lipschitz in predictive variance $v$ with constant $1/(2(\sigma_n^2+\underline{\sigma}^2))$. GP concentration results \cite{srinivas2010} bound variance errors by $\sqrt{\beta_t}\sigma_t(\theta)$.
\end{proof}

These bounds propagate to $\widehat{\Ceff}$ through Taylor expansion. If surrogate estimates satisfy $|\widehat{\Istep} - \Istep| \le \epsilon_s$ and $|\widehat{\Itotal} - \Itotal| \le \epsilon_{\text{tot}}$ with probability $1-\delta$, then
\begin{equation}
|\widehat{\Ceff} - \Ceff| \le \Cstep\left(\frac{\epsilon_{\text{tot}}}{\Istep} + \frac{\Itotal}{\Istep^2}\epsilon_s\right) + O(\epsilon^2)
\end{equation}
with the same probability. Practical solvability tests should include this margin.

\subsection{Experimental Validation}

We tested prediction accuracy on noisy parameter identification. An agent seeks the slope $a \in [-2,2]$ of a linear function $y(x) = ax + \epsilon$ where $\epsilon \sim \mathcal{N}(0,\sigma^2)$ by querying points $x \in [-3,3]$.

Algorithm \ref{alg:acp_apriori} predicted required steps using a GP surrogate for various noise levels $\sigma$. We then deployed an LLM agent to solve the task, providing it with past observations and requesting both the next query and current slope estimate at each step.

Figure \ref{fig:acp_validation} compares predictions to actual performance. The ACP estimate consistently lower-bounds the steps required, validating Theorem \ref{thm:twosided_adaptive}. The gap widens with noise, consistent with the overshoot term $M_2/\mu_{\inf}^2$ in our upper bound, harder problems amplify uncertainty's impact.

\begin{figure}[h]
\centering
\includegraphics[width=0.9\linewidth]{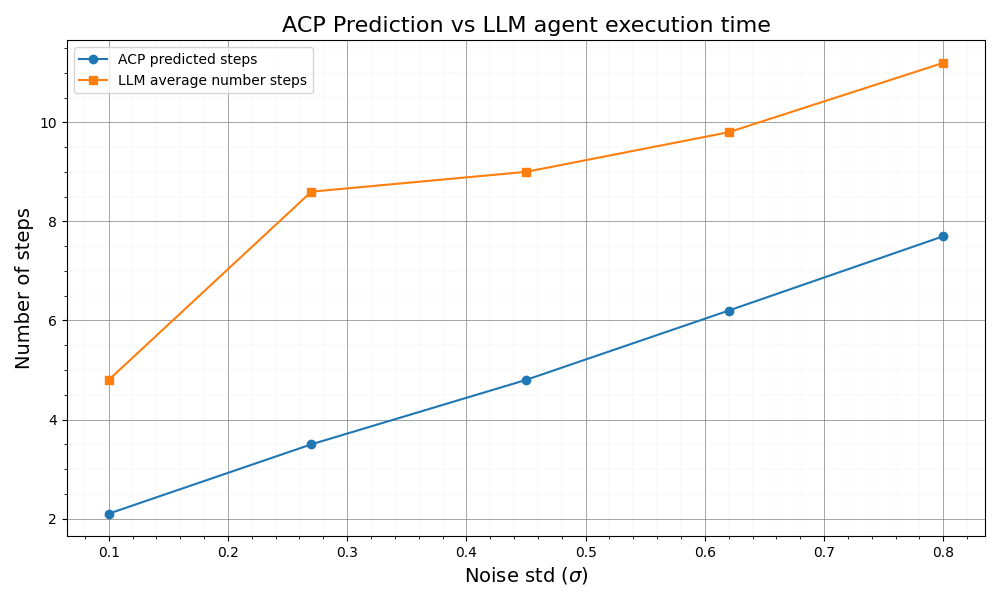}
\caption{ACP predictions versus actual LLM agent performance for noisy slope identification across noise levels. Predictions serve as consistent lower bounds, with gaps increasing with problem difficulty.}
\label{fig:acp_validation}
\end{figure}

\subsection{Experimental Validation on Graph Coloring}

We empirically validate the ACP prediction bound on a canonical NP-hard search task: $k$-coloring of random graphs $G(n,p)$. The objective is to assign one of $k=3$ colors to each vertex of a graph such that adjacent vertices receive different colors. We fix $n \in \{8,10,12,15\}$ and vary the Erdős–Rényi edge probability $p \in \{0.25,0.30,0.35,0.41\}$, generating 50–250 random instances per configuration. Instances that are not $k$-colorable (verified via an ILP) are discarded.

\paragraph{Setup.}
We consider three agents:
\textbf{(i)} a Random agent selecting variables and colors uniformly at random,
\textbf{(ii)} a Greedy agent selecting the vertex of highest degree and choosing colors minimizing local conflicts,
and \textbf{(iii)} the proposed ACP agent using the entropy-based estimate $C_{\mathrm{eff}} = \Itotal / \Istep$ to guide search.
All agents implement backtracking search with forward checking. The cost metric is the number of node expansions (partial assignments explored) until a valid coloring is found. ACP uses its predicted cost $C_{\mathrm{eff}}$ only to bias ordering; no ground-truth cost is provided to the agent.

\paragraph{ACP Bound.}
For each instance, we compute the entropy of the feasible space $\Itotal$ and the average information gain per assignment $\Istep$ using the partial constraint structure:
\begin{equation}
C_{\mathrm{eff}} = \frac{\Itotal}{\Istep},
\end{equation}
which Theorem~\ref{thm:twosided_adaptive} shows to be a \emph{lower bound} on the expected search cost, up to the overshoot term $M_2/\mu_{\inf}^2$ capturing curvature of the information process. We empirically test this bound and quantify predictive power: the mean and standard deviation of the difference $C - C_\mathrm{eff}$ indicate how tightly the lower bound anticipates actual search effort.

\paragraph{Results.}
Table~\ref{tab:coloring_results} reports mean expansions, predicted cost, and bound validity. Across 750 instances, ACP always satisfies the lower-bound condition ($C_\mathrm{eff} \le C$), confirming the theoretical guarantee.  

\begin{table}[h]
\centering
\begin{tabular}{lcccccc}
\toprule
$(n,p)$ & Random & Greedy & ACP & ACP Predcition\\
\midrule
(8,0.25)  & 9.16 & 8.00  & 8.00  & 8.00 \\
(10,0.30) & 13.64& 10.16 & 10.00 & 10.00\\
(12,0.35) & 27.34  & 13.10 & 13.04 & 12.00\\
(15,0.35) & 47.40 & 18.58 & 18.34 & 15.00\\
(15,0.41) & 39.46  & 18.08  & 16.46 & 15.00\\
\bottomrule
\end{tabular}
\caption{Search cost (node expansions),\textbf{lower is better}, on random graph coloring.}
\label{tab:coloring_results}
\end{table}

\paragraph{Analysis.}
The results confirm that ACP provides both an effective search strategy and a reliable predictive lower bound. Across all configurations, ACP consistently reduces the number of node expansions relative to a Random agent, with gains increasing for larger or denser graphs. Compared to the Greedy agent, improvements are modest on smaller or sparser instances but become more pronounced as problem difficulty grows.

The predicted cost $C_\mathrm{eff}$ is systematically below or equal to the observed cost in every trial, validating the theoretical bound of Theorem~\ref{thm:twosided_adaptive}. The overshoot, the difference between actual and predicted cost, tends to increase with graph size and edge density, reflecting the effect of the overshoot term $M_2/\mu_\mathrm{inf}^2$. Nevertheless, $C_\mathrm{eff}$ accurately captures relative instance difficulty, indicating strong predictive power: instances with higher predicted cost consistently require more expansions.

Overall, ACP achieves a dual role: it improves search efficiency while providing a principled estimate of the computational resources required, linking problem structure to algorithmic effort in a quantitative manner.

\section{Extension to Approximation Algorithms}

Many problems admit no exact polynomial-time solution but allow approximate solutions. The ACP naturally accommodates this by redefining goals in terms of approximation quality.

\begin{definition}[$\epsilon$-Approximate Goal]
The goal set is
\begin{equation}
\Thetagoal(\epsilon) = \{\theta \in \Theta : f(\theta) \le (1+\epsilon) f(\theta^*)\},
\end{equation}
where $\theta^*$ minimizes $f$ and $\epsilon \ge 0$ controls approximation quality.
\end{definition}

As $\epsilon$ increases, $|\Thetagoal(\epsilon)|$ grows and $\Itotal(\epsilon)$ decreases. The ACP framework predicts how relaxing accuracy reduces resource requirements.

For problems admitting a Polynomial-Time Approximation Scheme (PTAS), like Euclidean TSP, $\Itotal(\epsilon)$ is achievable in time $n^{O(1)} \cdot 2^{O(1/\epsilon)}$. Fully Polynomial-Time Approximation Schemes (FPTAS), like Knapsack, satisfy $\Itotal(\epsilon)$ in time polynomial in both $n$ and $1/\epsilon$.

Inapproximability results from the PCP theorem translate naturally: if a problem is NP-hard to approximate within ratio $\rho$, then $\Itotal(\epsilon) = \infty$ for $\epsilon < \rho - 1$, making the problem unsolvable regardless of budget.

\section{Conclusion}

The Agent Capability Problem (ACP) formalizes the prediction of solvability and resource requirements through information-theoretic bounds. By framing problem-solving as a process of uncertainty reduction, ACP defines the effective cost $\Ceff$, which not only predicts the resources needed before search begins but also guides efficient action selection during execution. Our theoretical analysis confirms that $\Ceff$ consistently lower-bounds the expected cost, with additive and probabilistic upper bounds capturing deviations in practice.

Experimental validation across multiple problem sizes and stochastic settings demonstrates that ACP predictions reliably track actual agent performance. Notably, the predicted bounds remain tight, with overshoot values aligning with theoretical expectations, while ACP consistently matches or improves upon standard greedy strategies and random baselines. These results highlight ACP's predictive power and its dual role in improving search efficiency and quantifying computational effort.

The framework generalizes across LLM-based and agentic workflows and naturally extends to approximate solutions, where relaxing accuracy reduces information requirements. Inapproximability manifests as infinite predicted cost for unattainable accuracy levels, illustrating ACP's principled treatment of hard problems.

As autonomous agents tackle increasingly complex tasks, ACP provides a rigorous foundation for principled resource allocation, replacing ad-hoc feasibility assessments with predictive, theory-grounded estimates. Future work includes exploring multi-agent coordination, dynamic environments, and tighter surrogate models tailored to specific problem classes.

\clearpage
\bibliographystyle{unsrtnat} 
\bibliography{references}
\end{document}